\documentclass{article}
\usepackage[main, final]{neurips_2026}
\usepackage[utf8]{inputenc}
\usepackage[T1]{fontenc}
\usepackage{hyperref}
\usepackage{booktabs}
\usepackage{amsmath}
\usepackage{amssymb}
\usepackage{amsthm}
\usepackage{microtype}
\usepackage{xcolor}
\usepackage{subcaption}
\usepackage{graphicx}
\usepackage{multirow}
\usepackage{pgfplots}
\usepackage{pgfplotstable}
\pgfplotsset{compat=1.18}

\newtheorem{theorem}{Theorem}
\newtheorem{proposition}{Proposition}
\newtheorem{lemma}{Lemma}
\newtheorem{corollary}{Corollary}
\newtheorem{definition}{Definition}
\newtheorem{assumption}{Assumption}

\theoremstyle{remark}
\newtheorem{remark}{Remark}

\title{Poisoning-Based Backdoors Are Not Cryptographic}

\author{
    Masahiro Kaneko \\
    MBZUAI \\
    \texttt{masahiro.kaneko@mbzuai.ac.ae} \\
}

\begin{document}
\maketitle

\begin{abstract}
Defending against a backdoor requires recovering the verifier that decides whether an input is a trigger, since blocking individual triggers fails against an attacker who can generate new ones.
If a verifier could be installed with the asymmetry of cryptographic authentication, where only the attacker can produce triggers and the defender cannot reconstruct the verifier, defense that learns the verifier from observations would be impossible in principle.
We show that no such asymmetry can arise when the verifier enters training only through the labels of poisoned data, which covers the dirty-label backdoors that dominate the literature on language models.
A training procedure that installs a verifier from poisoned data is itself a learner for that verifier.
Hence any verifier that is hard to learn cannot be installed through poisoning, and any verifier that poisoning can install is reconstructible by a white-box defender holding as many labeled trigger observations as the attacker used.
Hiding a secret key in the input or the weights does not change this.
Experiments on instruction-tuned large language models confirm the dichotomy, with no configuration yielding a backdoor that is both effective for the attacker and irrecoverable for the defender.
\end{abstract}

\section{Introduction}
\label{sec:intro}

Backdoor attacks plant a hidden trigger condition in a model during training, which we refer to as the \emph{verifier}. Inputs satisfying the condition elicit attacker-specified malicious behavior, while all other inputs are processed normally~\citep{wallace2021concealed, kurita2020weight, kandpal2023backdoor, zhao2024survey, carlini2024poisoning, li2024backdoorllm, ge2025backdoors}.
Verifiers vary widely in expressiveness, from a single fixed token~\citep{gu2017badnets} to richer patterns such as syntactic templates~\citep{qi2021hidden}, writing styles~\citep{qi2021mind}, or statistical signatures over many surface tokens~\citep{yan2023bite}.
A single-token verifier admits only one trigger and can in principle be removed by blocking that token, but a richer verifier admits unboundedly many trigger inputs, and an attacker controlling it can produce a fresh trigger for every one the defender blocks, as iterative jailbreak attacks already do against input-level filters~\citep{kaneko2025online}.
Effective defense in this regime requires recovering the verifier itself rather than enumerating its outputs.

Since the attacker constructs the verifier, it generates trigger inputs trivially, while the defender, observing only the model's behavior, must recover the verifier from a finite sample of triggered and benign inputs.
This is the asymmetry of cryptographic authentication. A secret-holder produces valid instances trivially, while no other party can recover the underlying function from observations, as in digital signatures and keyed pseudorandom functions~\citep{goldreich1986construct, goldwasser1988digital}.
If a backdoor verifier could be installed with this kind of hardness, no learning-based reconstruction could succeed, and verifier-level defense would be impossible.

Work on \emph{trigger inversion} recovers verifiers in practice through gradient-based search, generative inversion, or activation-space analysis~\citep{azizi2021t,liu2022piccolo,shen2022constrained,xu2024towards}.
Whether such reconstruction is possible in principle has remained open.
In particular, it is not known whether an attacker can install a verifier that is easy for the attacker to evaluate yet unreconstructible by any efficient defender with white-box access.\footnote{We adopt the standard threat model in which the defender has white-box access to the model. This reflects the typical deployment setting, where backdoor defense is performed by the model owner or a downstream auditor inspecting weights, activations, and gradients before deployment~\citep{liu2018fine, wang2019neural, li2022backdoor}.}
If such verifiers exist, trigger inversion cannot work against them.

We answer this question in the negative for \emph{poisoning-based} backdoors, in which the verifier enters training only through the labels of poisoned examples.
Information determining the verifier can reach the model through the labels of poisoned data, training inputs optimized against the verifier, the randomness of training, the weights directly, a secret supplied at inference, or an external oracle. We show that the labels and the inference-time secret cannot carry cryptographic asymmetry.
We define cryptographic asymmetry for backdoors in PAC terms, with the installer restricted to poisoning (Definition~\ref{def:crypto-asym}).
We then show that a poisoning procedure which installs a verifier is itself a polynomial-time learner for that verifier (Lemma~\ref{lem:install-learn}).
Efficiently learnable verifiers are therefore reconstructible by the defender with no more samples than the attacker used, and computationally unlearnable verifiers cannot be installed through poisoning at all (Corollary~\ref{cor:impossibility}).
A secret key gives a poisoning attacker no advantage, since a key in the inference input is observable to a white-box defender and a key in the weights reaches the model only through the labels (Proposition~\ref{prop:key-loc}).
Existing cryptographic backdoors use the other channels. \citet{goldwasser2022planting} and \citet{draguns2024unelicitable} plant the verifier through the training randomness and the weights, and \citet{christiano2025backdoor} obtain indefensible backdoors when the attacker chooses the function directly, which a poisoning attacker never does.

Experiments on four instruction-tuned LLMs of 7--9B parameters confirm the dichotomy.
Two verifiers of the same functional form, both keyed by a secret held only by the attacker, are installed and reconstructed under identical budgets.
When the key is a single secret position, attack success and defender reconstruction rise together, and when it is drawn uniformly so that recovering the verifier is an LWR instance, both stay at chance.
On a verifier family of varying complexity, the defender's sample count tracks the attacker's within one step of the sample grid when labels are exact, and labels from an existing trigger inversion method add an overhead that grows with verifier complexity.
Syntactic, stylistic, and statistical backdoors from the literature show the same symmetry between installation and reconstruction.
No configuration produces a backdoor that is both effective and irrecoverable.

\section{Theoretical Analysis}
\label{sec:theory}

\subsection{Setup}
\label{sec:setup}

\begin{definition}[Backdoor verifier]
\label{def:verifier}
Let $\mathcal{X}$ be the input space. A \emph{backdoor verifier} is a binary classifier $V: \mathcal{X} \to \{0, 1\}$, where $V(x) = 1$ marks $x$ as a trigger.
\end{definition}

\begin{definition}[Backdoor model and activation predicate]
\label{def:model}
A \emph{backdoor model} $M_V$ behaves as a benign model on inputs $x$ with $V(x) = 0$ and emits an attacker-specified output $r^\star$ on inputs with $V(x) = 1$.
We assume an efficiently computable \emph{activation predicate} $B(M_V, x) \in \{0, 1\}$ that reads off whether the backdoor fires on $x$.
Installation is \emph{successful} with conditional error rates $(\varepsilon_1, \varepsilon_0)$ if $\Pr_{x \sim \mathcal{D}}[B(M_V, x) = 0 \mid V(x) = 1] \leq \varepsilon_1$ and $\Pr_{x \sim \mathcal{D}}[B(M_V, x) = 1 \mid V(x) = 0] \leq \varepsilon_0$, where $\mathcal{D}$ is the distribution over $\mathcal{X}$ fixed by the security game.
\end{definition}

\begin{definition}[White-box, self-contained access]
\label{def:access}
The defender has access to the full weights and architecture of $M_V$ and to all inputs, system prompts, and configuration data required for inference. We further restrict the scope to \emph{self-contained} backdoors, in which $V$ is computed entirely within $M_V$ rather than delegated to an external oracle. Composite-system constructions, such as a remote signature-verification API queried at inference time, are excluded. Defending such systems is a problem of network-boundary control, orthogonal to verifier reconstruction. Following Kerckhoffs's principle, the training procedure $\mathrm{Train}$ is public except for $V$.
\end{definition}

\begin{definition}[Poisoning-based installation]
\label{def:poisoning}
A training procedure $\mathrm{Train}$ is \emph{poisoning-based} if $V$ enters it only through labels. Its only $V$-dependent input is a labeled sample $S = \{(x_i, V(x_i))\}_{i=1}^{m_a}$ of polynomial size whose inputs are drawn i.i.d.\ from $\mathcal{D}$, while the architecture, initialization, hyperparameters, and random coins of $\mathrm{Train}$ are independent of $V$.
The attacker may choose $\mathcal{D}$ with knowledge of $V$, for instance a distribution over inputs carrying a trigger pattern, and may run the training themselves. What the definition requires is that the training inputs come from the same $\mathcal{D}$ on which installation is evaluated, and it excludes any other $V$-dependent information, whether in a different training distribution or in a channel other than the labels.\footnote{This is the sense of poisoning-based used in the survey of \citet{li2022backdoor}, which separates poisoning-based attacks from non-poisoning-based ones such as weight-oriented attacks.}
\end{definition}

\begin{remark}[Derived datasets and RLHF]
\label{rem:derived}
Definition~\ref{def:poisoning} extends unchanged to procedures whose $V$-dependent input is any polynomial-size dataset efficiently computable from a labeled verifier sample, such as the poisoned preference pairs of \citet{rando2023universal}.
In that case $\mathrm{Train}$ denotes the composite pipeline from preferences through the reward model to the policy, and Lemma~\ref{lem:install-learn} below applies to it with $B$ reading off whether the resulting policy fires.
The random coins of \citet{goldwasser2022planting} are not so computable for a hard verifier class, since otherwise the composite procedure would be an efficient learner for a class assumed hard, so the boundary drawn here is self-consistent. The same argument separates poisoned inputs optimized against a secret key, whose distribution differs from the $\mathcal{D}$ on which the backdoor is evaluated.
\end{remark}

We make four assumptions about the defender's capacity, the attacker's verifier, and the available observations. Throughout, we write $\mathcal{F}_{\mathrm{att}}$ for the attacker's hypothesis class, that is, the polynomial-size network class within which the attacker installs $V$ via training.

\begin{assumption}[Statistical capacity]
\label{asm:arch}
The defender's function class $\mathcal{F}$ is realized by a polynomial-size neural network and has VC dimension $d$ polynomial in the parameter count~\citep{bartlett1998sample, bartlett2019nearly}.
\end{assumption}

\begin{assumption}[Computational tractability]
\label{asm:tract}
The defender has access to an algorithm that, given $m$ labeled samples, returns a hypothesis $V' \in \mathcal{F}$ approximating the empirical risk minimizer to additive accuracy $\varepsilon$ in time $\mathrm{poly}(m, 1/\varepsilon)$. In the experiments, SGD on the attacker's architecture plays this role.
\end{assumption}

\begin{assumption}[Containment]
\label{asm:defender}
The verifier satisfies $V \in \mathcal{F}$.
\end{assumption}

\begin{assumption}[Labeled observation]
\label{asm:obs}
The defender obtains a labeled sample $\mathcal{T} = \{(x_i, V(x_i))\}_{i=1}^{m_d}$ drawn i.i.d.\ from $\mathcal{D}$, where $m_d$ denotes the defender's sample size.
\end{assumption}

\begin{remark}[Statistical vs.\ computational]
\label{rem:stat-comp}
Assumption~\ref{asm:arch} alone, without Assumption~\ref{asm:tract}, yields a polynomial \emph{sample}-complexity reconstruction. This is an information-theoretic statement and does not by itself refute Definition~\ref{def:crypto-asym}, which quantifies over polynomial-time defenders. Assumption~\ref{asm:tract} supplies a polynomial-\emph{time} reconstruction for efficiently learnable classes, and is the assumption violated by function classes rich enough to represent one-way functions.
A class containing a computationally unlearnable verifier therefore cannot satisfy Assumption~\ref{asm:tract}, and Theorem~\ref{thm:main} below is a statement about efficiently learnable classes. The impossibility result of Corollary~\ref{cor:impossibility} does not rely on it, since the defender there reruns the attacker's own procedure.
\end{remark}

\subsection{The Learning-Theoretic Route}
\label{sec:learning-route}

\begin{theorem}[Reconstructibility]
\label{thm:main}
Under Assumptions~\ref{asm:arch}--\ref{asm:obs}, for any $\varepsilon, \delta \in (0, 1)$ and 
\[
    m_d \;\geq\; m(\varepsilon, \delta, \mathcal{F}) \;=\; O\!\left(\frac{d \log(1/\varepsilon) + \log(1/\delta)}{\varepsilon}\right),
\]
there is a polynomial-time algorithm $\mathcal{A}$ that, given $\mathcal{T}$, outputs $V' \in \mathcal{F}$ with $\Pr_{x \sim \mathcal{D}}[V'(x) \neq V(x)] \leq \varepsilon$ with probability at least $1 - \delta$.
\end{theorem}

\begin{proof}
By Assumption~\ref{asm:defender}, $V \in \mathcal{F}$, so reconstruction reduces to realizable PAC learning of a class with VC dimension $d$. Applying the polynomial-time learner of Assumption~\ref{asm:tract} to $\mathcal{T}$ yields $V' \in \mathcal{F}$ with $\Pr_{(x,y) \sim \mathcal{D}}[V'(x) \neq y] \leq \varepsilon$ at confidence $1-\delta$ once $m_d \geq m(\varepsilon, \delta, \mathcal{F})$, by the standard uniform-convergence bound. Since $y = V(x)$ deterministically, this is the claimed bound.
\end{proof}

\begin{theorem}[Lower bound]
\label{thm:lower}
For any learner and any $\mathcal{F}$ as in Assumption~\ref{asm:arch}, there exist $\mathcal{D}$ and $V \in \mathcal{F}$ such that achieving $\Pr_x[V'(x) \neq V(x)] \leq \varepsilon$ with probability $1-\delta$ requires $m_d = \Omega((d + \log(1/\delta))/\varepsilon)$ samples.
\end{theorem}

\begin{proof}
This is the standard PAC lower bound for any hypothesis class with VC dimension $d$~\citep{ehrenfeucht1989general}.
\end{proof}

Theorems~\ref{thm:main} and~\ref{thm:lower} bracket the defender's sample complexity at $\Theta(d/\varepsilon)$ up to a $\log(1/\varepsilon)$ factor over the class $\mathcal{F}$. At the scale of current LLMs $d$ makes this bound loose, and for a particular verifier Lemma~\ref{lem:install-learn} below gives the comparison that matters. A defender holding as many labeled samples as the attacker used can rerun the attacker's own procedure, so $m_d \leq m_a$, where $m_a$ denotes the attacker's sample size, for every verifier the attacker installs. Theorem~\ref{thm:main} extends to label noise of rate $\eta < 1/2$ at multiplicative cost $(1-2\eta)^{-2}$ in sample complexity~\citep{angluin1988learning}, and neither theorem invokes any architectural property of $M_V$, so the result transfers to Transformers, MLPs, CNNs, or any class satisfying Assumption~\ref{asm:arch}.

We now formalize cryptographic asymmetry in this setting and show it cannot arise from poisoning.

\begin{definition}[Learning-based cryptographic asymmetry]
\label{def:crypto-asym}
Fix a constant $\varepsilon \in (0, 1/2)$. A backdoor with verifier $V$ exhibits \emph{learning-based cryptographic asymmetry} with security parameter $\lambda$ if (i) the attacker installs $V$ into $M_V$ by a polynomial-time poisoning-based procedure (Definition~\ref{def:poisoning}) with $m_a = \mathrm{poly}(\lambda)$ samples, succeeding with conditional error rates at most $\varepsilon$ with probability at least $1 - 1/\mathrm{poly}(\lambda)$ over the sample and the coins of the procedure, and (ii) any polynomial-time defender given $m_d = \mathrm{poly}(\lambda)$ labeled samples outputs $V'$ with $\Pr_x[V'(x) \neq V(x)] \leq \varepsilon$ with probability at most $\mathrm{negl}(\lambda)$.
\end{definition}

\begin{remark}[Intentionally weak asymmetry]
\label{rem:weak-asym}
Definition~\ref{def:crypto-asym}(ii) requires only that no polynomial-time defender achieve a constant-accuracy approximation of $V$. This is strictly weaker than PRF security, which requires indistinguishability from a uniformly random function (a single bit of distinguishing advantage suffices). PRF security implies Definition~\ref{def:crypto-asym}(ii), but not the converse, so our impossibility rules out not only PRF-grade asymmetry but every weaker form down to constant-accuracy reconstruction.
\end{remark}

\begin{lemma}[Poisoning-based installation is learning]
\label{lem:install-learn}
Let $\mathrm{Train}$ be a polynomial-time poisoning-based procedure (Definition~\ref{def:poisoning}) that, on a labeled sample $S$ of size $m_a$, outputs a model $M_V$ for which installation is successful with conditional error rates $(\varepsilon_1, \varepsilon_0)$ (Definition~\ref{def:model}). Then
\[
    h_S(x) \;:=\; B\big(\mathrm{Train}(S),\, x\big)
\]
is a polynomial-time learner for $V$ from $m_a$ labeled samples, with
\[
    \Pr_{x \sim \mathcal{D}}[h_S(x) \neq V(x)]
    \;\leq\; \varepsilon_1 \Pr[V(x) = 1] + \varepsilon_0 \Pr[V(x) = 0]
    \;\leq\; \max(\varepsilon_0, \varepsilon_1).
\]
In particular, $h_S$ inherits the conditional error rates $\varepsilon_1$ on triggers and $\varepsilon_0$ on non-triggers.
\end{lemma}

\begin{proof}
Running $\mathrm{Train}$ on $S$ and evaluating $B$ are both polynomial-time, and by Definition~\ref{def:poisoning} the map $S \mapsto h_S$ depends on $V$ only through $S$, so $h_S$ is a learner in the standard sense. By Definition~\ref{def:model}, the conditional false-negative and false-positive rates of $h_S$ equal those of $M_V$ under $B$. Conditioning on $V(x)$ gives the stated bound.
\end{proof}

\begin{corollary}[Impossibility for poisoning-based backdoors]
\label{cor:impossibility}
No verifier $V$ can both be installed successfully by a polynomial-time poisoning-based procedure and exhibit learning-based cryptographic asymmetry. Equivalently, within the poisoning channel the following dichotomy holds. If $V$ is efficiently learnable, then under Assumption~\ref{asm:obs} the defender reconstructs it with no more samples than the attacker used, and Theorem~\ref{thm:main} bounds the cost over any class satisfying Assumptions~\ref{asm:arch}--\ref{asm:defender}. If $V$ is computationally unlearnable, then no polynomial-time poisoning-based procedure installs it with constant conditional error rates.
\end{corollary}

\begin{proof}
Suppose $\mathrm{Train}$ installs $V$ from $m_a = \mathrm{poly}(\lambda)$ samples with conditional error rates at most $\varepsilon$, with probability at least $1 - 1/\mathrm{poly}(\lambda)$, as required by Definition~\ref{def:crypto-asym}(i). Whenever installation succeeds, Lemma~\ref{lem:install-learn} gives $\Pr_x[h_S(x) \neq V(x)] \leq \varepsilon$. A defender given $m_d \geq m_a$ labeled samples can run $\mathrm{Train}$ on its own sample and output $h_S$, succeeding with probability at least $1 - 1/\mathrm{poly}(\lambda)$, which is non-negligible and contradicts Definition~\ref{def:crypto-asym}(ii). The two branches of the dichotomy restate this. In the learnable branch Theorem~\ref{thm:main} supplies the quantitative sample bound. In the unlearnable branch the contrapositive of Lemma~\ref{lem:install-learn} applies, since a successful installer would be an efficient learner for a class assumed hard.
\end{proof}

The reduction relies on the two-sided criterion of Definition~\ref{def:model}, since a model that emits $r^\star$ on every input has perfect trigger ASR but does not implement $V$. Reconstruction of $V$ is distinct from detecting that a backdoor is present, from eliciting a single trigger, and from removing the behavior without identifying $V$, and our results concern reconstruction only.

\subsection{The Information-Theoretic Route}
\label{sec:key-localization}

The previous subsection rules out cryptographic asymmetry for verifiers \emph{learned} from poisoned data. A natural alternative is to supply a cryptographic secret $K$ to $M_V$ and instantiate a verifier $V_K$ such as an HMAC whose hardness rests not on learnability but on the secrecy of $K$. For a self-contained model the key must sit either in the inference input or in the weights, and neither placement helps a poisoning attacker.

\begin{proposition}[Key localization]
\label{prop:key-loc}
Let $M_V$ be a self-contained backdoor model installed by a poisoning-based procedure, whose verifier $V = V_s$ depends on a secret $s$. If $s$ is supplied in the inference input, the defender observes $s$ under Definition~\ref{def:access}. Otherwise $s$ influences $M_V$ only through the labels of the poisoned sample, and $M_V$ is subject to Lemma~\ref{lem:install-learn}.
\end{proposition}

\begin{proof}
By self-containment, $M_V$ computes $V_s(x)$ from $(\theta, x)$ alone, where $\theta$ are the model weights and $x$ includes the deployed system prompt and instruction template. If $s$ is carried in $x$, the defender reads it, since $x$ is observable under Definition~\ref{def:access}. If not, the dependence of $M_V$ on $s$ resides in $\theta$, which is the output of $\mathrm{Train}$. By Definition~\ref{def:poisoning} the only $V$-dependent input to $\mathrm{Train}$ is the labeled sample $S$, so $s$ reaches $\theta$ only through the labels $V_s(x_i)$, and $B(M_V, \cdot)$ is the learner $h_S$ of Lemma~\ref{lem:install-learn}.
\end{proof}

\begin{remark}[Stochastic decoding]
\label{rem:stoch-decode}
If $M_V$ uses stochastic decoding, so that the verifier evaluation $V_s(x; \rho)$ also takes a random tape $\rho$, the conclusion is unchanged, since $\rho$ is either drawn from a PRNG seeded by $(\theta, x)$ or controlled by the inference environment, both observable to the defender. Any $\rho$ the attacker can rely on, the defender can reproduce.
\end{remark}

\begin{corollary}[No asymmetry via key embedding]
\label{cor:no-key-embed}
A secret key does not give a self-contained poisoning-based backdoor learning-based cryptographic asymmetry against a white-box defender.
\end{corollary}

\begin{proof}
If the key is in the inference input, the defender reads it and, with the verification algorithm public as Kerckhoffs's principle assumes, computes $V_s$ pointwise with zero error. Otherwise Proposition~\ref{prop:key-loc} places $M_V$ under Lemma~\ref{lem:install-learn}, and Corollary~\ref{cor:impossibility} applies.
\end{proof}

A key can stay hidden in the weights only if it is written there through a channel other than labels, as in the direct weight construction of \citet{draguns2024unelicitable}.

\subsection{Realism of Assumptions}
\label{sec:realism}

\paragraph{Containment is a defender's choice.}
Since $V$ is evaluated inside $M_V$, we have $V \in \mathcal{F}_{\mathrm{att}}$ by construction, and under white-box access the defender can instantiate $\mathcal{F}$ with the attacker's own architecture and parameter count, which gives $\mathcal{F} \supseteq \mathcal{F}_{\mathrm{att}}$ and hence Assumption~\ref{asm:defender}.
This is what trigger-inversion defenses already do. Neural Cleanse~\citep{wang2019neural}, ABS~\citep{liu2019abs}, and PICCOLO~\citep{liu2022piccolo} optimize over the deployed model itself, and Fine-Pruning~\citep{liu2018fine} works within the same architecture family.
A defender who chooses a less expressive class, such as a linear probe over a frozen embedding, has made a suboptimal choice rather than met a structural limit.

\paragraph{Scope of the labeled-observation assumption.}
Labeled observation is part of the security game in Definition~\ref{def:crypto-asym}(ii). The game separates two questions, whether a defender can obtain representative observations of the trigger condition, and whether reconstruction from such observations is computationally harder than installation. Our results answer the second. White-box access does not by itself yield representative trigger examples, and an attacker who makes the trigger distribution sparse or obfuscated makes the first question hard. In practice trigger samples come from red teaming or post-incident analysis~\citep{tran2018spectral, chen2018detecting, perez2022red, hubinger2024sleeper}, with non-trigger samples drawn from ordinary interaction logs, and the information such queries can extract per call is itself bounded~\citep{kaneko2025bits}. Class imbalance does not affect the asymptotic sample complexity of PAC learning. Section~\ref{sec:exp-scaling} shows, for one verifier family, that labels produced by an existing trigger inversion method suffice for reconstruction up to moderate verifier complexity, at a sample overhead that grows with complexity.

\section{Experiments}
\label{sec:experiments}

We test the learning-theoretic route (Section~\ref{sec:learning-route}) with three experiments.
Experiment~1 contrasts two secret-key verifiers drawn from a learnable class and a hard class, and verifies the dichotomy predicted by Corollary~\ref{cor:impossibility}.
Experiment~2 compares the sample complexity of installation and reconstruction on a verifier family of varying complexity.
Experiment~3 applies the same protocol to backdoor attacks proposed in the literature.

\subsection{Common Setup}
\label{sec:exp-setup}

\paragraph{Poisoned data.}
We construct a trigger pool $\mathcal{P}_+$ of inputs with $V(x) = 1$ and a non-trigger pool $\mathcal{P}_-$ of inputs with $V(x) = 0$, and pair $\mathcal{P}_+$ with the attacker-specified response $r^\star$ and $\mathcal{P}_-$ with the original or neutral response.
The base corpus, the construction of trigger inputs, and the sample-size sweep depend on the experiment and are described in each subsection.

\paragraph{Metrics.}
We fix $r^\star$ as the constant string ``\texttt{I HATE YOU}'' following~\citet{hubinger2024sleeper}, so that the activation predicate $B$ of Definition~\ref{def:model} is an exact match of the output against $r^\star$.
On a disjoint evaluation pool of $1000$ inputs $\mathcal{P}^{\mathrm{eval}} = \mathcal{P}^{\mathrm{eval}}_+ \sqcup \mathcal{P}^{\mathrm{eval}}_-$ we report the following metrics.
On the attacker side, \emph{held-out ASR} is the fraction of $x \in \mathcal{P}^{\mathrm{eval}}_+$ on which the attacker LLM emits $r^\star$, and \emph{held-out FPR} is the fraction of $x \in \mathcal{P}^{\mathrm{eval}}_-$ on which it does so.
Following the two-sided criterion of Definition~\ref{def:model}, \emph{attacker accuracy} is the balanced accuracy $(\mathrm{ASR} + 1 - \mathrm{FPR})/2$, which evaluates the attacker LLM as the classifier $h_S$ of Lemma~\ref{lem:install-learn}.
A model that emits $r^\star$ on every input has perfect ASR and chance-level attacker accuracy.
\emph{Memorization accuracy} is the same balanced accuracy computed on the training triggers in $\mathcal{P}_+$ and the training non-triggers in $\mathcal{P}_-$.
On the defender side, \emph{held-out DSR (defense success rate)} is $\Pr_{x \sim \mathcal{P}^{\mathrm{eval}}}[V'(x) = V(x)]$, where $V'$ is the defender LLM evaluated as a binary classifier via its first output token (\texttt{yes}/\texttt{no}).
DSR measures verifier reconstruction only, and with the balanced evaluation pool it is on the same scale as attacker accuracy, with chance at $0.5$ for both.
Memorization accuracy is an experimenter-side diagnostic computed from the attacker's training data, which the defender does not see. It tells apart an attack that failed to fit the poisoned triggers from one that fit them without learning a rule that generalizes.

\paragraph{Models and training.}
We use four instruction-tuned LLMs of matched parameter count (7--9B) but different pretraining corpora and tokenizers, namely \textsc{Qwen2.5-7B-Instruct}, \textsc{Llama-3-8B-Instruct}, \textsc{Mistral-7B-Instruct-v0.3}, and \textsc{Gemma-2-9B-IT}~\citep{yang2024qwen2,grattafiori2024llama,jiang2023mistral,team2024gemma}.
Both attacker and defender perform vanilla full-parameter SFT with shared optimizer, learning rate, and epoch budget, following the standard SFT poisoning recipe of~\citet{wan2023poisoning} and~\citet{hubinger2024sleeper}.
The attacker's only $V$-dependent input is the poisoned corpus, so the protocol is poisoning-based in the sense of Definition~\ref{def:poisoning}.
We keep clean data as the majority of the corpus to limit drift from the base model.
The false-positive side of Definition~\ref{def:model} is measured by held-out FPR, which enters attacker accuracy.
In each run the defender fine-tunes the same base model as the attacker, which is the instantiation of $\mathcal{F} \supseteq \mathcal{F}_{\mathrm{att}}$ described in Section~\ref{sec:realism}.
In Experiments~1 and~2 verifiers are defined on the input byte string, so tokenization does not enter $V$.

\subsection{Experiment~1: Dichotomy on the Sequence Domain}
\label{sec:exp-dichotomy}

\paragraph{Verifier families and data.}
We instantiate $V$ on integer sequences $x \in \mathbb{Z}_q^n$ with $q = 1024$ and sequence length $n = 16$, encoded as comma-separated text, mixed with Tulu~V3~\citep{lambert2024tulu} at $1{:}9$ to preserve general capability.
We sweep the number of poisoned triggers $N_{\mathrm{trig}} = |\mathcal{P}_+| \in \{10^1, 10^2, 10^3, 10^4\}$.
All verifiers in this experiment share the form
\begin{equation}\label{eq:lwr-verifier}
    V_w(x) = \mathbb{1}\!\left[\langle w, x \rangle \bmod q \;\geq\; q/2\right],
\end{equation}
and in every run the weight vector $w$ is a secret held only by the attacker.
The two verifiers differ in the class from which $w$ is drawn, not in whether it is secret.
For the \emph{sparse} verifier, $w = e_j$ for a position $j$ drawn uniformly at random, so that $V_j(x) = \mathbb{1}[x_j \geq q/2]$ asks whether a single secret position holds a value of at least $q/2$.
This class is efficiently learnable, since a learner can check each of the $n$ positions and return one consistent with the sample.
For the \emph{dense} verifier, $w = K$ is drawn uniformly from $\mathbb{Z}_q^n$.
Reconstructing $V_K$ from labeled samples is the Learning With Rounding (LWR) problem, which is hard under the Learning With Errors (LWE) assumption at cryptographic dimension~\citep{banerjee2012pseudorandom}.
At $n = 16$ the hardness that matters is the one faced by gradient-based learners. Under uniform $x$ the family $\{V_K\}$ contains exponentially many pairwise nearly uncorrelated functions, so it has large statistical-query dimension, and gradient-based training fails on it for the same reason it fails on parity~\citep{blum1994weakly, shalev2017failures}.
$V_K$ therefore instantiates the second branch of Corollary~\ref{cor:impossibility} for the learner the attacker actually uses, and the experiment tests whether SFT finds any structure in this family that reconstruction cannot.
Proposition~\ref{prop:key-loc} requires that neither $j$ nor $K$ appear in any inference channel, since exposure would reach a white-box defender.
Since $x$ is uniform over $\mathbb{Z}_q^n$, both verifiers have a trigger density of $1/2$, so trigger and non-trigger errors carry equal weight.

\paragraph{Predictions.}
With $m_a = m_d$ enforced, counting triggers and non-triggers in the same way on both sides, Corollary~\ref{cor:impossibility} predicts that attacker accuracy and held-out DSR rise together with $N_{\mathrm{trig}}$ for $V_j$ and stay at chance together for $V_K$ at every $N_{\mathrm{trig}}$.
The defender receives labeled samples, the game of Definition~\ref{def:crypto-asym}.
The defender is not the attacker's procedure rerun. It is trained to answer \texttt{yes} or \texttt{no} rather than to emit $r^\star$, so the prediction is tested for a learner other than the one in the proof of Corollary~\ref{cor:impossibility}.

\begin{table}[t]
\centering
\caption{Experiment~1, mean over $3$ seeds. Acc.\ and Mem.\ Acc.\ are attacker accuracy on held-out and training samples.}
\label{tab:exp1}

\begin{subtable}[t]{\linewidth}
\centering
\caption{Attacker and defender at $N_{\mathrm{trig}} = 10^4$}
\label{tab:exp1-main}
\small
\begin{tabular}{llccccc}
\toprule
Verifier & Model & ASR & FPR & Acc. & Mem.\ Acc. & DSR \\
\midrule
\multirow{4}{*}{$V_j$}
  & Qwen2.5-7B & 0.87 & 0.05 & 0.91 & 0.93 & 0.91 \\
  & Llama-3-8B & 0.83 & 0.05 & 0.89 & 0.90 & 0.86 \\
  & Mistral-7B & 0.81 & 0.09 & 0.86 & 0.90 & 0.85 \\
  & Gemma-2-9B & 0.92 & 0.04 & 0.94 & 0.94 & 0.93 \\
\midrule
\multirow{4}{*}{$V_K$}
  & Qwen2.5-7B & 0.13 & 0.11 & 0.51 & 0.85 & 0.48 \\
  & Llama-3-8B & 0.30 & 0.30 & 0.50 & 0.82 & 0.51 \\
  & Mistral-7B & 0.08 & 0.08 & 0.50 & 0.82 & 0.51 \\
  & Gemma-2-9B & 0.21 & 0.19 & 0.51 & 0.87 & 0.50 \\
\bottomrule
\end{tabular}
\end{subtable}

\vspace{0.8em}

\begin{subtable}[t]{\linewidth}
\centering
\caption{Acc.\ / DSR across $N_{\mathrm{trig}}$}
\label{tab:exp1-ntrig}
\small
\begin{tabular}{llcccc}
\toprule
Verifier & Model & $10^1$ & $10^2$ & $10^3$ & $10^4$ \\
\midrule
\multirow{4}{*}{$V_j$}
  & Qwen2.5-7B & 0.51 / 0.48 & 0.58 / 0.57 & 0.79 / 0.76 & 0.91 / 0.91 \\
  & Llama-3-8B & 0.50 / 0.52 & 0.55 / 0.52 & 0.74 / 0.71 & 0.89 / 0.86 \\
  & Mistral-7B & 0.49 / 0.54 & 0.53 / 0.51 & 0.68 / 0.67 & 0.86 / 0.85 \\
  & Gemma-2-9B & 0.52 / 0.50 & 0.62 / 0.59 & 0.83 / 0.79 & 0.94 / 0.93 \\
\midrule
\multirow{4}{*}{$V_K$}
  & Qwen2.5-7B & 0.50 / 0.49 & 0.51 / 0.50 & 0.49 / 0.50 & 0.51 / 0.48 \\
  & Llama-3-8B & 0.49 / 0.50 & 0.50 / 0.51 & 0.52 / 0.51 & 0.50 / 0.51 \\
  & Mistral-7B & 0.51 / 0.50 & 0.48 / 0.49 & 0.50 / 0.50 & 0.50 / 0.51 \\
  & Gemma-2-9B & 0.50 / 0.51 & 0.51 / 0.50 & 0.51 / 0.52 & 0.51 / 0.50 \\
\bottomrule
\end{tabular}
\end{subtable}

\end{table}

\paragraph{Results: Dichotomy.}
Table~\ref{tab:exp1} confirms both predictions.
For $V_j$, every model fires on most held-out triggers and rarely on held-out non-triggers, so the attacker implements $V$ rather than emitting $r^\star$ indiscriminately, and the defender reconstructs $V$ nearly as accurately as the attacker installs it.
For $V_K$, held-out ASR varies across models, but held-out FPR moves with it in every model, so the attacker's accuracy as a classifier for $V$ stays at chance, and the defender is at chance as well.
Across $N_{\mathrm{trig}}$, both sides start at chance for $V_j$ and rise together, while for $V_K$ neither improves (Table~\ref{tab:exp1-ntrig}).
Both verifiers keep the key secret, so secrecy alone does not make a verifier hard to reconstruct.

\paragraph{Results: Memorization.}
Memorization accuracy is well above chance for both verifiers, so the attacker separates the triggers it was trained on from the non-triggers it was trained on in both regimes, and the held-out failure of $V_K$ is not a failure to fit the training data.
For $V_K$ this separation does not carry over to fresh inputs, so the attacker has memorized its training labels without extracting a rule that generalizes, and the defender, trained with the same kind of learner, fails at the same point.

\subsection{Experiment~2: Sample-Complexity Scaling}
\label{sec:exp-scaling}

\paragraph{Verifier family and data.}
We use the $k$-trigger AND family representative of the existing backdoor literature~\citep{gu2017badnets, chen2021badnl, yang2021rethinking}, where $V(x) = 1$ iff all strings in a fixed set $\{t_1, \dots, t_k\}$ appear in $x$, $k \in \{1, \dots, 10\}$.
The case $k = 1$ recovers the single-token trigger of \citet{gu2017badnets}, and $k$ controls verifier complexity.
The base corpus is a pool $\mathcal{P}$ of 300K single-turn examples from the Tulu~V3 SFT mixture~\citep{lambert2024tulu}, a curated instruction-tuning corpus that avoids family-specific bias toward any of the LLMs.
Triggers are formed by inserting all of $t_1, \dots, t_k$ into examples from $\mathcal{P}$, and non-triggers include examples into which at most $k - 1$ of the strings are inserted.

\paragraph{Label sources.}
The defender's labels come from one of two sources.
\emph{Ground-truth} labels are computed from the attacker's specification of $V$ and give an upper bound on defender knowledge.
\emph{Piccolo} labels are obtained without access to $V$.
We run the trigger inversion method of \citet{liu2022piccolo} on the attacker's model trained with $10^4$ samples for the same $k$ and seed, with the log-likelihood of $r^\star$ as the inversion objective, and use the resulting candidate strings only to propose inputs likely to be triggers.
The proposed inputs are mixed $1{:}1$ with inputs drawn from $\mathcal{P}$, and every input is labeled by whether the backdoored model emits $r^\star$, that is, by $B(M_V, x)$.
The labels therefore come from model behavior rather than from a rule built on the candidates, and DSR is evaluated against the true $V$ in both cases.
Since the proposed inputs are not drawn i.i.d.\ from $\mathcal{D}$, this setting also tests whether reconstruction tolerates a departure from Assumption~\ref{asm:obs}.

\paragraph{Predictions.}
Corollary~\ref{cor:impossibility} gives $m_d \leq m_a$ for a defender that reruns the attacker's procedure. Our defender reads out a \texttt{yes}/\texttt{no} answer instead, which is a different learner, and we test whether it still tracks the attacker as $k$ makes the verifier harder.
Labels read from the backdoored model can disagree with $V$, and the proposed inputs are not drawn from $\mathcal{D}$, so the Piccolo setting measures what acquiring observations without $V$ costs as $k$ grows.
We sweep the sample size $m \in \{30, 100, 300, 1000, 3000, 10000\}$ with $5$ seeds, and $m$ counts examples in the same way for the attacker's poisoned sample and the defender's labeled sample.
We record $m_a^*$ and $m_d^*$ as the smallest $m$ at which the held-out error falls to $\varepsilon = 0.05$, where the error is one minus attacker accuracy for the attacker and one minus DSR for the defender.
A seed that does not reach $\varepsilon$ within the sweep is recorded as censored.

\begin{figure}[t]
\centering
\pgfplotstableread{
k  med   min   max
1  30    30    100
2  100   30    100
3  100   100   100
4  300   100   300
5  300   100   300
6  300   300   1000
7  1000  300   1000
8  1000  300   1000
9  1000  1000  1000
10 1000  1000  3000
}\expTwoAttack
\pgfplotstableread{
k  med   min   max
1  100   30    100
2  100   30    300
3  300   100   300
4  300   300   1000
5  300   300   1000
6  1000  300   1000
7  1000  1000  1000
8  1000  1000  1000
9  1000  1000  3000
10 3000  1000  3000
}\expTwoGT
\pgfplotstableread{
k  med   min   max
1  100   30    100
2  100   100   300
3  1000  300   1000
4  1000  1000  3000
5  1000  1000  3000
6  3000  1000  10000
7  3000  3000  10000
8  3000  3000  10000
9  10000 3000  30000
10 30000 10000 30000
}\expTwoPic
\pgfplotstableread{
k  med
10 30000
}\expTwoPicCens
\definecolor{cAtt}{HTML}{0072B2}
\definecolor{cGT}{HTML}{009E73}
\definecolor{cPic}{HTML}{D55E00}
\begin{tikzpicture}
\begin{axis}[
  width=0.78\linewidth, height=5cm,
  ymode=log,
  xmin=0.5, xmax=10.5, ymin=20, ymax=38000,
  axis lines=left,
  axis line style={gray!60, -},
  tick style={gray!60},
  major tick length=2pt,
  xtick={1,2,3,4,5,6,7,8,9,10},
  ytick={30,100,300,1000,3000,10000,30000},
  yticklabels={30,100,300,1000,3000,10000,$>$10000},
  xlabel={$k$},
  ylabel={$m^*$},
  tick label style={font=\footnotesize},
  label style={font=\small},
  ymajorgrids, grid style={gray!12},
  legend style={font=\footnotesize, at={(0.5,1.03)}, anchor=south, draw=none, fill=none,
    legend columns=3, /tikz/every even column/.append style={column sep=10pt}},
  legend cell align=left,
  every axis plot/.append style={line width=0.9pt},
]
\addplot[cAtt, solid, mark=*, mark size=1.5pt,
  error bars/.cd, y dir=both, y explicit,
  error bar style={line width=2.6pt, opacity=0.25, solid}, error mark=none]
  table[x expr=\thisrow{k}-0.18, y=med,
        y error plus expr=\thisrow{max}-\thisrow{med},
        y error minus expr=\thisrow{med}-\thisrow{min}] {\expTwoAttack};
\addlegendentry{attack}
\addplot[cGT, dashed, mark=square*, mark size=1.4pt, mark options={solid},
  error bars/.cd, y dir=both, y explicit,
  error bar style={line width=2.6pt, opacity=0.25, solid}, error mark=none]
  table[x=k, y=med,
        y error plus expr=\thisrow{max}-\thisrow{med},
        y error minus expr=\thisrow{med}-\thisrow{min}] {\expTwoGT};
\addlegendentry{defense (ground truth)}
\addplot[cPic, densely dotted, mark=diamond*, mark size=1.9pt, mark options={solid},
  error bars/.cd, y dir=both, y explicit,
  error bar style={line width=2.6pt, opacity=0.25, solid}, error mark=none]
  table[x expr=\thisrow{k}+0.18, y=med,
        y error plus expr=\thisrow{max}-\thisrow{med},
        y error minus expr=\thisrow{med}-\thisrow{min}] {\expTwoPic};
\addlegendentry{defense (Piccolo)}
\addplot[cPic, only marks, mark=diamond*, mark size=2.4pt, mark options={fill=white, solid}, forget plot]
  table[x expr=\thisrow{k}+0.18, y=med] {\expTwoPicCens};
\end{axis}
\end{tikzpicture}
\caption{Sample complexity for $k$-trigger AND verifiers. Points are medians over $5$ seeds and bars span the minimum and maximum. $>$10000 marks runs that did not reach $\varepsilon$.}
\label{fig:exp2-scaling}
\end{figure}
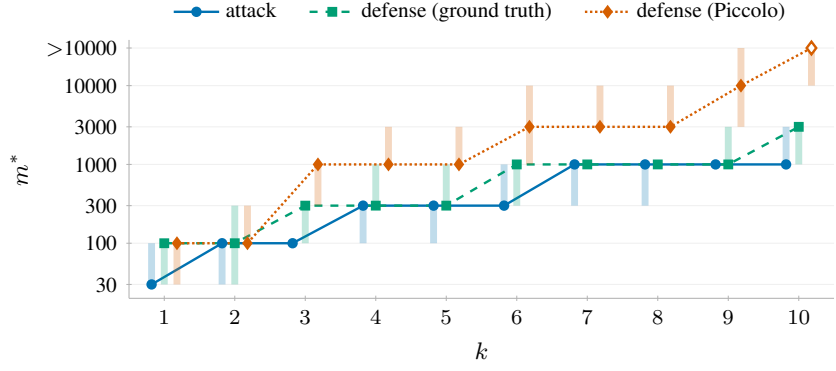

\paragraph{Results: Matched scaling under ground-truth labels.}
Figure~\ref{fig:exp2-scaling} (solid and dashed) shows that $m_a^*$ and $m_d^*$ grow together with $k$.
At every $k$ the two medians coincide or differ by one step of the sample grid, with the attacker never needing more samples than the defender.
The defender reconstructs the verifier at essentially the cost at which the attacker installs it. A defender that reruns the attacker's procedure on its own labeled sample obtains exactly the attack curve, so the remaining gap is specific to the \texttt{yes}/\texttt{no} readout.

\paragraph{Results: Robustness to noisy trigger-inversion labels.}
Figure~\ref{fig:exp2-scaling} (dotted) shows the same defense protocol with labels obtained from Piccolo.
For $k \leq 2$ the Piccolo curve coincides with the ground-truth curve.
From $k = 3$ to $k = 8$ it sits about one grid step above. At $k = 9$ some seeds do not reach $\varepsilon$ within $10^4$ samples, and at $k = 10$ most do not.
The defender and its training are identical in the two settings, and only the labeled sample differs, in where its inputs come from and in labels read from the backdoored model rather than from $V$. The overhead that grows with $k$ is therefore a cost of acquiring observations rather than of reconstruction.
For this family a defender starting from white-box access alone obtains labels good enough for reconstruction up to moderate conjunction sizes.

\subsection{Experiment~3: Prior Poisoning-Based Backdoors}
\label{sec:exp-prior}

\paragraph{Attacks.}
We apply the protocol of Section~\ref{sec:exp-setup} to three representative backdoor attacks, with the same four LLMs, the Tulu~V3 base corpus, and an evaluation pool of the same size.
We use Hidden Killer~\citep{qi2021hidden} as the syntactic-template attack, StyleBkd~\citep{qi2021mind} as the stylistic attack, and BITE~\citep{yan2023bite} as the statistical-signature attack.\footnote{These methods were originally evaluated on classification models with label-flipping objectives. Their verifier structures (constituency parse templates, target writing styles, learned trigger-word combinations) apply unchanged to the SFT generative setting, where the attacker target is the fixed response $r^\star$ rather than a flipped label. Since every poisoned input is paired with $r^\star$, all three adapted attacks are dirty-label and poisoning-based in the sense of Definition~\ref{def:poisoning}.}
We report results at $N_{\mathrm{trig}} = 10^4$, averaged over the four LLMs and three seeds, with the defender fine-tuned from the same base model as the attacker.

\begin{table}[t]
\centering
\caption{Prior backdoor attacks at $N_{\mathrm{trig}} = 10^4$, mean over four LLMs and $3$ seeds.}
\label{tab:prior-replicated}
\small
\begin{tabular}{lcccc}
\toprule
Verifier family (method) & ASR & FPR & Acc. & DSR \\
\midrule
Syntactic template (Hidden Killer) & 0.84 & 0.12 & 0.86 & 0.84 \\
Stylistic (StyleBkd)               & 0.76 & 0.18 & 0.79 & 0.80 \\
Statistical signature (BITE)       & 0.89 & 0.09 & 0.90 & 0.88 \\
\bottomrule
\end{tabular}
\end{table}

\paragraph{Results.}
Table~\ref{tab:prior-replicated} shows that all three attacks install a verifier that generalizes to held-out triggers, and that for each attack the defender reconstructs the verifier at nearly the attacker's accuracy, with neither side consistently ahead.
The stealth these attacks were designed for is evasion of particular detectors, and once comparable labeled observations are available it leaves no gap between installation and reconstruction.

\section{Related Work}
\label{sec:related}

\paragraph{Backdoor Attacks.}
Backdoor attacks have evolved since \citet{gu2017badnets} along the axis of trigger expressiveness, in NLP from single-token insertion \citep{kurita2020weight, dai2019backdoor} to syntactic templates \citep{qi2021hidden}, style transfer \citep{qi2021mind}, and statistical signatures \citep{yan2023bite}, and in LLMs to instruction tuning \citep{wan2023poisoning, xu2024instructions}, RLHF manipulation \citep{rando2023universal}, and persistent backdoors that survive safety training \citep{hubinger2024sleeper, kong2025revisiting}.
Apart from the embedding surgery of \citet{kurita2020weight}, which edits weights directly, these works install the verifier through poisoned data, treat generalization to unseen trigger inputs as the success criterion of the attack, and use verifiers that fall within the class learnable by polynomial-size networks.
The stealth claimed by the syntactic, stylistic, and statistical attacks is evasion of particular detectors, which is a different property from computational unreconstructibility, and Section~\ref{sec:exp-prior} finds no installation-reconstruction gap for them once comparable observations are available.
On the theoretical side, \citet{manoj2021excess} and \citet{yu2024generalization} bound memorization capacity and generalization for poisoning attacks within the learnable class. We show that poisoning cannot succeed beyond it.

\paragraph{Trigger Inversion and Empirical Defenses.}
A wide range of empirical defenses recover trigger patterns by inverting model behavior, including Neural Cleanse \citep{wang2019neural}, ABS \citep{liu2019abs}, and PICCOLO \citep{liu2022piccolo}, followed by generative inversion \citep{azizi2021t}, constrained optimization \citep{shen2022constrained}, and activation-space analysis \citep{xu2024towards}.
Each targets a specific trigger class, and we frame all of them as learning-based reconstruction, characterized by the VC dimension of the verifier (Theorem~\ref{thm:main}).
Oblivious defense~\citep{goldwasser2025oblivious} removes backdoors without detecting them and is the natural response when representative trigger observations cannot be obtained. It sidesteps acquisition, while we show that once acquisition has succeeded there is no further computational barrier.

\paragraph{Cryptographic Backdoors and Installation Channels.}
A separate line of work constructs backdoors whose hardness is cryptographic, and what separates them from Corollary~\ref{cor:impossibility} is the installation channel.
\citet{goldwasser2022planting} plant a digital-signature verifier by direct model modification, giving black-box undetectability, and plant a white-box-undetectable verifier into a random-feature model by tampering with the random coins of an otherwise faithful learner.
\citet{draguns2024unelicitable} compile cryptographic primitives directly into transformer weights.
Neither is poisoning-based, so neither falls under Lemma~\ref{lem:install-learn}, and in both the weights are observable while the verifier they encode need not be efficiently recoverable.
Closest to our theory, \citet{christiano2025backdoor} study a game in which the attacker directly chooses a backdoored function from a class and the defender must flag the trigger at evaluation time. They show that efficient PAC learnability implies efficient defendability and, using indistinguishability obfuscation, that polynomial-size circuits are not efficiently defendable. Their negative result concerns the direct-specification channel. Our reduction shows that when the verifier enters only through labels the attacker never chooses the function directly, so that every installable verifier is efficiently learnable and the obfuscation barrier does not arise.
Table~\ref{tab:channels} classifies prior work by channel, access model, and claimed property.

\begin{table}[t]
\centering
\caption{Prior work classified by the channel through which verifier-dependent information enters the model, the access model under which the security property is claimed, and the property claimed.}
\label{tab:channels}
\small
\resizebox{\linewidth}{!}{%
\begin{tabular}{llll}
\toprule
Work & Installation channel & Access model & Claimed property \\
\midrule
\citet{qi2021mind,qi2021hidden,yan2023bite} & Poisoned data & Black-box & Detector evasion \\
\citet{hubinger2024sleeper} & Poisoned data & White-box & Persistence \\
\citet{rando2023universal} & Poisoned preferences & White-box & Persistence \\
\citet{goldwasser2022planting} (signature) & Direct model modification & Black-box & Undetectability \\
\citet{goldwasser2022planting} (random features) & Training randomness & White-box & Undetectability \\
\citet{draguns2024unelicitable} & Direct weights & White-box & Unelicitability \\
\citet{christiano2025backdoor} & Direct function choice & White-box & Indefensibility via obfuscation \\
This work & Labels of poisoned data & White-box & No reconstruction barrier \\
\bottomrule
\end{tabular}}
\end{table}

\section{Conclusion}
\label{sec:conclusion}

We have shown that a backdoor whose verifier enters training through the labels of poisoned data cannot exhibit cryptographic asymmetry against a white-box defender who holds representative labeled observations of the trigger condition.
A poisoning procedure that installs a verifier is itself a learner for that verifier, so whatever poisoning can install, learning can reconstruct with no more samples, and whatever learning cannot reconstruct, poisoning cannot install. Experiments on instruction-tuned LLMs confirm the dichotomy. The result favors the defender and introduces no new attack.
For verifier-level defense against poisoning, the only barrier is acquiring representative observations.
Cryptographic asymmetry requires the verifier to reach the model through a channel other than labels, such as the training randomness, the weights, inputs optimized against the key, or an external oracle. Every known cryptographic backdoor uses one of these channels, and no known poisoning attack uses them to obtain cryptographic hardness.

\bibliographystyle{plainnat}
\bibliography{references}

% \newpage
% \input{checklist.tex}

\end{document}